\documentclass[journal]{IEEEtran}

\usepackage{algorithm}
\usepackage{amsmath}
\usepackage{amsfonts}
\usepackage{bm}
\usepackage{amssymb}
\usepackage{algorithmic}
\usepackage{graphicx}
\usepackage{amsthm}
\usepackage{mathtools}
\usepackage{subfig}
\usepackage{array}
\usepackage{multirow}
\usepackage{cite}
\usepackage{times}
\usepackage{epsfig}
\usepackage{color}
\usepackage{cases}
\usepackage{sidecap}
\usepackage{wrapfig}
\usepackage{float}
\usepackage{leftidx}
\usepackage{pifont}

\newtheorem{thm}{Theorem}
\DeclareMathOperator{\softmax}{softmax}

\makeatletter
\DeclareRobustCommand\onedot{\futurelet\@let@token\@onedot}
\def\@onedot{\ifx\@let@token.\else.\null\fi\xspace}

\def\etal{\emph{et al}\onedot}
\makeatother

\newcommand{\Tref}[1]{Table~\ref{#1}}
\newcommand{\Eref}[1]{equation~(\ref{#1})}
\newcommand{\Erefs}[1]{equations~(\ref{#1})}
\newcommand{\Fref}[1]{Figure~\ref{#1}}

\newcommand{\bx}{\bm{x}}
\newcommand{\bo}{\bm{o}}
\newcommand{\bz}{\bm{z}}
\newcommand{\ba}{\bm{a}}
\newcommand{\bdelta}{\bm{\delta}}

\definecolor{myGreen}{rgb}{0, .8, .3}
\definecolor{myRed}{rgb}{0.8, .2, .2}

\begin{document}

\title{Robust Student Network Learning}

\author{Tianyu Guo, Chang Xu, Shiyi He, Boxin Shi,~\IEEEmembership{Member,~IEEE}, Chao Xu, and Dacheng Tao,~\IEEEmembership{Fellow,~IEEE}
\thanks{Tianyu Guo, Shiye He and Chao Xu are with the Key Laboratory of Machine Perception (Ministry of Education) and Coopertative Medianet Innovation Center, School of EECS, Peking University, Beijing 100871, P.R. China. E-mail: tianyuguo@pku.edu.cn, shiyihe@pku.edu.cn, xuchao@cis.pku.edu.cn.}
\thanks{Boxin Shi is with the National Engineering Laboratory for Video Technology, School of EECS, Peking University, Beijing 100871, P.R. China. E-mail: shiboxin@pku.edu.cn.}
\thanks{Chang Xu and Dacheng Tao are with the UBTech Sydney Artificial Intelligence Centre and the School of Information Technologies in the Faculty of Engineering and Information Technologies at The University of Sydney, J12 Cleveland St, Darlington NSW 2008, Australia. Email: c.xu@sydney.edu.au, dacheng.tao@sydney.edu.au.}}

\markboth{}%
{Shell \MakeLowercase{\etal}: Robust student network learning}

\maketitle

\begin{abstract}
Deep neural networks bring in impressive accuracy in various applications, but the success often relies on the heavy network architecture. Taking well-trained heavy networks as teachers, classical teacher-student learning paradigm aims to learn a student network that is lightweight yet accurate. In this way, a portable student network with significantly fewer parameters can achieve a considerable accuracy which is comparable to that of teacher network. However, beyond accuracy, robustness of the learned student network against perturbation is also essential for practical uses. 
Existing teacher-student learning frameworks mainly focus on accuracy and compression ratios, but ignore the robustness. In this paper, we make the student network produce more confident predictions with the help of the teacher network, and analyze the lower bound of the perturbation that will destroy the confidence of the student network. Two important objectives regarding prediction scores and gradients of examples are developed to maximize this lower bound, so as to enhance the robustness of the student network without sacrificing the performance. Experiments on benchmark datasets demonstrate the efficiency of the proposed approach to learn robust student networks which have satisfying accuracy and compact sizes.
\end{abstract}


\IEEEpeerreviewmaketitle

\section{Introduction}

\IEEEPARstart{R}{ecent} years have witnessed the marked progress of deep learning. Since the breakthrough in 2012 ImageNet competition~\cite{russakovsky2015imagenet} achieved by AlexNet~\cite{krizhevsky2012imagenet} using five convolutional layers and three fully connected layers, a series of more advanced deep neural networks have been developed to keep rewriting the record, \emph{e.g.}, VGGNet~\cite{simonyan2014very}, GoogLeNet~\cite{szegedy2016rethinking}, and ResNet~\cite{he2016deep}. However, their excellent performance requires the support from a huge amount of computation.  For instance, AlexNet~\cite{simonyan2014very} contains about 232 million parameters and needs $7.24 \times 10^8$ multiplications to process an image with resolution of $227 \times 227$. Hence, the potential power of deep neural networks can only be fully unlocked on high performance GPU servers or clusters. In contrast, majority of the mobile devices used in our daily life usually have rigorous constraints on the storage and computational resource, which prevents them from fully taking advantages of deep neural network. As a result, networks with smaller hardware demanding while still maintaining similar accuracies are of great interests to image processing and computer vision community.
  
Compressing convolutional neural networks can be achieved by vector quantization~\cite{gong2014compressing}, decomposing weight matrices~\cite{denton2014exploiting}, and encoding with hashing tricks~\cite{chen2015compressing}.  Unimportant weights can be pruned to achieve the same goal by removing the subtle weights~\cite{han2015learning,han2015deep}, reducing the redundancy between weights in the frequency domain~\cite{wang2016cnnpack}, and using the binary networks~\cite{courbariaux2016binarized,rastegari2016xnor}. Another straightforward approach is to design a compact network directly, \emph{e.g.}, ResNeXt~\cite{xie2017aggregated}, Xception network~\cite{chollet2017xception}, and MobileNets~\cite{howard2017mobilenets}. These networks are often deep and thin with fewer parameters in each layer, and the non-linearity of these networks are strengthened by increasing the number of layers, which guarantees the performance of the network.

Student-teacher learning framework, introduced in knowledge distillation (KD)~\cite{hinton2015distilling}, is one of the most popular approaches to realize model compression and acceleration~\cite{courbariaux2016binarized,wang2016cnnpack}. Taking a heavy neural network, such as GoogleNet~\cite{szegedy2016rethinking} or ResNet~\cite{he2016deep}, that has already been well trained with massive data and computing resources as the teacher network, a student network of light architecture can be better learned under teacher's guidance. To inherit the advantages of teacher networks, different methods have been proposed to encourage the consistency between teacher and student network. For example, Ba and Caruana~\cite{NIPS2014_5484} minimized the Euclidean distance between features extracted from these two networks, Hinton \emph{et al.}~\cite{hinton2015distilling} encouraged the student to mimic a softened version of the teacher's output, and FitNet~\cite{romero2014fitnets} introduced intermediate-level hints from teacher's hidden layers to guide the training process of student. Patrick and Nikolaus~\cite{mcclure2016representational} proposed to keep the pairwise distance of examples between student network and teacher. You \emph{et al.}~\cite{you2017learning} utilized multiple teacher networks to guide the training process of student network. Wang \emph{et al.}~\cite{wangAAAI18} introduced a teaching assistant to encourage the similarity between distributions of features maps extracted from teacher and student networks.
  
\begin{figure*}
\setlength{\abovecaptionskip}{0.3cm}
\setlength{\belowcaptionskip}{0.3cm}
\begin{minipage}[t]{1.0\linewidth}
\centering
\includegraphics[height=3.35in]{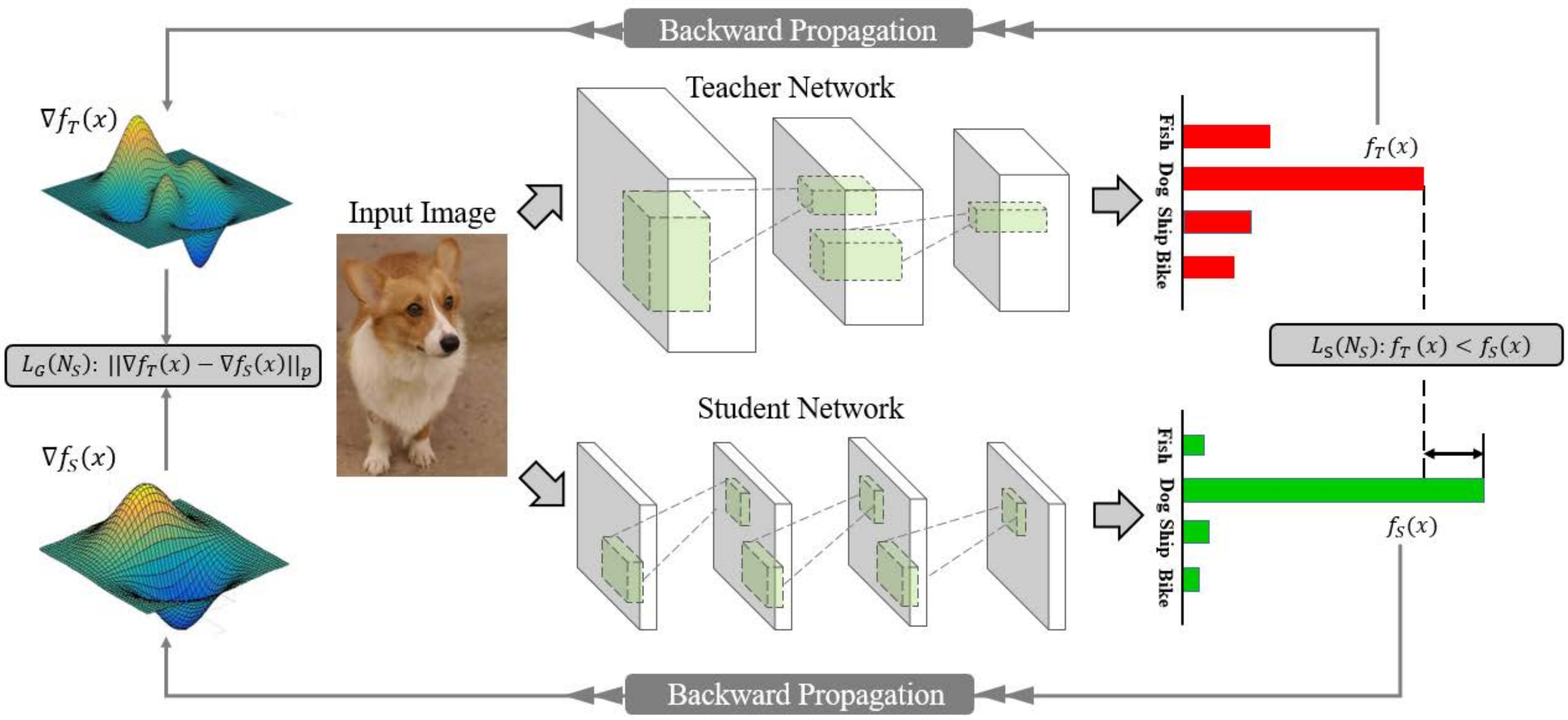}
\caption{
Framework of the proposed algorithm. Constraint $\mathcal L_S(\mathcal N_S)$ imposed on the output part ensures student network to have higher confidence in the prediction than that of teacher network. Constraint $\mathcal L_G(\mathcal N_S)$ imposed on the gradients encourages student network to preserve its confidence in the prediction if there is perturbation on the data. $f(\bx)$ represents the network's prediction for the ground-truth label $y$ of input $x$.}
\label{fig1}
\end{minipage}
\end{figure*}

These aforementioned algorithms have achieved impressive experimental results, however, they were mainly developed in ideal scenarios, where all data are implicitly assumed to be clean. In practice, given examples with perturbation, the training process of the network can be seriously influenced, and the resulting network would not be confident as before to make predictions of examples. Teacher network might make some mistakes, since it is difficult for teacher network to be familiar with all examples fed into the student network. This is consistent with student-teacher learning in the real world. An excellent student is expected to solve practical problems in changeable circumstances, where there might be questions even not known by teachers. 

To solve this problem, in this paper, we introduce a robust teacher-student learning algorithm. The framework of the proposed method is illustrated in \Fref{fig1}. We enable student network to be more confident on its prediction with the help of teacher network. Perturbations on examples might seriously influence the learning of student network. We derive the lower bound of the perturbations that can make student be more vulnerable than teacher through a rigorous theoretical analysis. New objectives in terms of prediction scores and gradients of examples are further developed to maximize the lower bound of the required perturbation. Hence, the overall robustness of the student network to resist perturbations on examples can be improved. Experimental results on benchmark datasets demonstrate the superiority of the proposed method for learning compact and robust deep neural networks.

We organized the rest of the paper as follows. In Section II, we summarize related works on learning convolutional neural networks with fewer parameters by different methods. Section III introduces the previous work we based on. In Section IV, we formally introduce our robust student network learning method in detail, including mathematical proof to the proposed theorem, the calculation method of loss function, and the training strategy. Section V provides results of our algorithm obtained on various benchmark datasets to prove the effectiveness of the proposed method. Section VI concludes this paper.

\section{Related works}
In this section, we briefly introduce related works on learning a efficient convolutional neural networks with fewer parameters. There are two different categories of methods according to their techniques and motivations.

\subsection{Network Trimming}
  
Network trimming aims to remove redundancy in heavy networks to obtain a compact network with fewer parameters and less computational complexity, whereas the accuracy of this portable network is close to that of the original large model.
Gong \emph{et al.}~\cite{gong2014compressing} utilized the benefits of vector quantization to compress neural networks, and a cluster center of weights was introduced as the representation of similar weights. 
Denton \emph{et al.}~\cite{denton2014exploiting} implemented singular value decomposition to the weight matrix of a fully-connected layer to reduce the number of parameters. 
Chen \emph{et al.}~\cite{chen2015compressing} attempted to explore hash encoding to improve the compression ratio. 
Courbariaux \emph{et al.}~\cite{courbariaux2016binarized} and Rastegari \emph{et al.}~\cite{rastegari2016xnor} implemented binary networks. All weights previously storied as 32-bit floating, are converted to binary ($\{-1, 0, 1\}$ or $\{-1, 1\}$).
Moreover, Wang \emph{et al.}~\cite{wang2016cnnpack} and Han \emph{et al.}~\cite{han2015learning} exploited weight pruning to achieve the same goal. In particular, Han \emph{et al.}~\cite{han2015learning} focused on removing subtle weights to reduce the parameters while minimizing the impact of removing them. Over 80\% subtle weights were dropped without the accuracy drop. Furthermore, Han \emph{et al.}~\cite{han2015deep} integrated several neural network compression techniques \emph{i.e.} pruning, quantization, and Huffman coding to further compress the network. Wang \emph{et al.}~\cite{wang2016cnnpack} showed that redundancy exists in not only subtle weights, but also large weights. It converted convolutional kernels into frequency domain to reduce the redundancy contained in larger weights and thereby compress networks with a higher compression ratio. In addition, Wang \emph{et al.}~\cite{wang2017beyond} focused on the redundancy in feature maps instead of network weights, which can also be considered as a modification of network architecture. Although the network trimming method brings a considerable compression and speedup ratio, due to the highly sparse parameters and the irregular network architectures, the actual acceleration effect is often heavily dependent on the customized hardware.

\subsection{Design Small Networks}
Directly designing a new deep neural network of light size is a straightforward approach to realize efficient deep learning. Most of these methods increase the depth of networks with much lower complexity compared to simply stacking convolution layers. For example, ResNet~\cite{he2016deep} introduced a novel residual block that obtained a significant performance with only slightly computation costs. ResNeXt~\cite{xie2017aggregated} explored group convolutions into the building blocks to boost performance. Flattened networks~\cite{jin2014flattened} introduced fully factorized convolutions and designed an extremely factorized network. Almost at the same time, Factorized Networks~\cite{wang2016factorized} introduced topological convolution that treats sections of tensors separately. SqueezeNet~\cite{iandola2016squeezenet} designed a portable network with a bottleneck architecture. SENet~\cite{hu2017squeeze} proposed a novel architecture named SE block, which focuses on the relationship between channels. Moreover,~\cite{vanhoucke2014learning} introduced depth-wise separable convolutions to obtain a great gain in the speed, and the size of networks. With the help from depth-wise separable convolutions, Inception models~\cite{szegedy2015going,normalization2015accelerating} reduced the complexity of the first few layers of network. Latter, Xception network~\cite{chollet2017xception} outperformed Inception model by scaling up depth-wise separable convolutional filters. Subsequently, the MobileNets~\cite{howard2017mobilenets} combined channel-wised decomposition of convolutional filters with depth-wise separable convolutions and achieved state-of-the-art results among portable models. ShuffleNet~\cite{zhang2017shufflenet} introduced a novel from of group convolution and depth-wise separable convolution. Deep fried convnets~\cite{yang2015deep} introduced a novel Adaptive Fastfood transform to reduce the computation of networks. Structured transform networks~\cite{NIPS2015_5869} offered considerable accuracy-compactness-speed tradeoffs based on the new notions rooted in the theory of structured matrices.

\subsection{Teacher-Student Learning}
There is another way to train a portable network. Regard the trained network as a teacher and the deeper yet thinner network as a student. With the help of the intrinsic information captured by the teacher network, the deeper and thinner student network could be well trained. Ba and Caruana~\cite{NIPS2014_5484} suggested that student network mimic the features extracted from the last layer of the teacher networks to assist the training progress of student networks, thereby increasing the depth of the student network. 
Knowledge Distillation (KD)~\cite{hinton2015distilling} pointed out that for two networks with huge structural differences, it is difficult to directly mimic features. Therefore, KD~\cite{hinton2015distilling} proposes to minimize the relaxed output of softmax layers of the two networks. This strategy can further deepen the student network. 
FitNet~\cite{romero2014fitnets}, based on KD, minimized the difference between the features extracted from the middle layers of student and teacher networks. They added several layers of MLPs at the middle layer of the teacher network to match the dimensions of the features of the student network. By establishing a connection between the middle layers of two networks, the student network can be further deepened with fewer parameters. 
McClure and Kriegeskorte~\cite{mcclure2016representational} attempted to minimize the distance between pairs of samples to reduce the difficulty of training students' networks. You \emph{et al.}~\cite{you2017learning} proposed utilizing multiple teacher networks to provide more guidance for the training of student networks. They leverage a voting strategy to balance the multiple guidance from each teacher network. Wang \emph{et al.}~\cite{wangAAAI18} regarded student network as a generator which is a part of GAN~\cite{NIPS2014_5423}, as well as utilized a discriminator as a assistant of teacher for forcing student to generating features which are difficult to distinguish from the features of teacher.

Compared to the network trimming algorithm, the student-teacher learning framework has more flexibility, no special requirements on hardware, and a more structured network structure. Compared to the direct design of a deeper network, guidance from the teacher is beneficial to learning deep networks and improving the performance of student. However, existing student-teacher algorithms pay more attention to improving the performance of student network on pure data sets. The instability caused by the large reduction in parameters makes the performance degradation under the Perturbation settings not yet studied. Therefore, a more robust learning algorithm for improving student network performance under perturbed conditions needs to be developed. This paper proposed a method under the teacher-student learning and knowledge distillation framework, which enhanced the robustness of student network.

\section{Preliminary of Teacher-Student Learning}
To make this paper self-contained, we briefly introduce some preliminary knowledge of teach-student learning here.

The teacher network $\mathcal N_T$ has complicated architecture, and it has already been well trained to achieve a sufficiently high performance. We aim to learn a student network $\mathcal N_S$, which is deeper yet thinner than the teacher network $\mathcal N_T$ but has a lower yet satisfying accuracy. Let $\mathcal X$ be the example space and $\mathcal Y$ be its corresponding $k$-label space. Outputs of these two networks are defined as: 
\begin{equation}
\begin{aligned}
\bo_T = \softmax(\ba_T),\ \ \ \ \ \ \bo_S = \softmax(\ba_S), \\
\end{aligned}
\label{eq:softmax}
\end{equation}
where $\ba_T$ and $\ba_S$ are the features produced by pre-softmax layers of teacher and student networks, respectively.

The teacher network $\mathcal N_T$ is usually trained on a relatively large dataset and consists of a large number of parameters, so that the teacher network usually achieves a high accuracy in classification task. Given significantly fewer parameters and numbers of multiplication operations, if adopting the same training strategies as the teacher network, the student network $\mathcal N_S$ is difficult to achieve a high performance. It is therefore necessary to, improve student network performance by investigating the assistance of the teacher network. A straightforward method is to encourage the features of an image extracted from these two networks to be similar~\cite{NIPS2014_5484}. The objection function can be written as 
\begin{equation}
\mathcal L(\mathcal N_S) = \frac{1}{n}\sum_{i=1}^n[\mathcal H(\bo_S^i,y^i) + \frac{\lambda}{2} \| \bo_S^i-\bo_T^i\|^2],
\label{eq:2}
\end{equation}
where the second term helps the student network to extract knowledge from the teacher, $\mathcal{H}$ refers to the cross-entropy loss, ,$\bo_S^i$ indicates the output of the $i$-th example in $\mathcal{X}$ by the student network, $y^i$ refers to the corresponding label, and $\lambda$ is the coefficient to balance two terms in the function.
The teacher and student networks can be significantly different in architecture, and thus it is difficult to expect features extracted by these two networks for the same example to be same. Hence, Knowledge Distillation (KD)~\cite{hinton2015distilling}, as an effective alternative, was proposed to distill knowledge from classification results to minimize
\begin{equation}
\mathcal L_{KD}(\mathcal N_S) = \frac{1}{n}\sum_{i=1}^n[\mathcal H(\bo_S^i,y^i) + \lambda \mathcal H(\bm \tau(\bo_S^i),\bm \tau(\bo_T^i))],
\label{eq:kdLoss}
\end{equation}
where the second term $\mathcal H(\bm \tau(\bo_S^i),\bm \tau(\bo_T^i))$ aims to enforce the student network to learn from softened output of the teacher network. $\bm\tau(\cdot)$ is a relaxation function defined as follow:
\begin{equation}
\begin{aligned}
\bm \tau(\bo_T) = \softmax(\frac{\ba_T}{\tau}), \\
\bm \tau(\bo_S) = \softmax(\frac{\ba_S}{\tau}).
\end{aligned}
\label{eq:tau}
\end{equation}
$\bm\tau$ is introduced to make sure that the second term in \Eref{eq:kdLoss} can play a different role compared with the first one. This is because that $\bo_T$ might be extremely similar to the one hot code representation of the ground-truth labels, while a soften version of output is different from the true labels. Moreover, the soften version of output could also provide more information to guide the learning of student, as the cross-entropy loss and soften version output will enhance the influence of classes other than the true label one.

Although KD loss in \Eref{eq:kdLoss} allows the student network to access the knowledge from the teacher network, the significant reduction in the number of parameters decreases the capability of the student network and makes it more vulnerable to input disturbances. The learned student network might achieve a reasonable performance on clean data, but it would suffer from a serious performance decline when encountering perturbation on the data in real world applications. To solve this issue, it is therefore necessary to enforce the robustness of the student network when applied to practical scenario.

\section{Robust Student Network Learning} 
We take a multi-class classification problem over $k$ classes as an example to introduce our robust Student Network Learning. Given a teacher network $\mathcal N_T$ and a student network $\mathcal N_S$, an example $\bx$ can then be classified by two networks $\bo_T(\bx) = \mathcal N_T(\bx)$ and $\bo_S(\bx) = \mathcal N_S(\bx)$, respectively. Denote $o_T^j(\bx)$ and  $o_S^j(\bx)$ as the $j$-th value of the $k$-dimensional vectors $\bo_T(\bx)$ and $\bo_S(\bx)$, respectively. Then we define $f_T(\bx) = o_T^y(\bx)$ and $f_S(\bx) = o_S^y(\bx)$ as the scores produced by two networks for the ground-truth label $y$ of the example $x$, respectively. If a classifier has more confidence in its prediction, the predicted score will be higher. With the help of the teacher network, the student network is supposed to be more confident in its prediction, so that
\begin{equation}
f_S(\bx)>f_T(\bx).
\label{eq:constrain1}
\end{equation}

\subsection{Theoretical Analysi}
The above relationship holds in ideally noise-free scenario. In practical scenario, perturbations on examples are unavoidable, and the student network is expected to resist the unexpected influence and bring in the robust prediction,
\begin{equation}
f_S(\bx+\bdelta)>f_T(\bx+\bdelta), \ \ \ \ \|\bdelta\|_2 \leq R \ \ \text{and}\ \ x+\bdelta \in C,
\label{eq:what_is_robust}
\end{equation}
where $\bdelta$ is a perturbation added to $x$. We restrict this perturbation in a spherical space of radius $R$, and $C$ is a constraint set that specifies some requirements for the input, \emph{e.g.,} an image input should be in $[0,1]^d$, where $d$ is the dimension. We define the ball as $B_p(\bx,R) = \{\bz\in\mathbb R^d\ |\ \|x-\bz\|_p \leq R\}$.

We aim to discover a student network that stands on the shoulder of the teachers to make a confident prediction not only for clean examples but also for examples with perturbations. The perturbation $\bdelta$ exists on examples without influencing their corresponding ground-truth labels. However, with the increase of perturbation intensity, the learning process of the student network would be seriously disturbed. Taking \Eref{eq:what_is_robust} as an auxiliary constraint in training the student network can be helpful for improving robustness of the network. But it is difficult and impossible to enumerate and try every possible $\bdelta$ to form the constraint. To make the optimization problem tractable, we seek for some alternatives and proceed to study the maximum perturbation that can be defended by the system. \Fref{fig1} shows the framework of our approach.

\begin{thm}
Let $x\in \mathbb R^d$ be an example in $\mathcal X$. $f_S(\bx)$ and $f_T(\bx)$ are functions adapted from the student and the teacher networks to predict the label $y$ of example $x$, respectively. Given $f_S(\bx) > f_T(\bx)$, for any $\bdelta\in \mathbb R^d$ with
\begin{equation}
\|\bdelta\|_q\geq 
\frac{f_S(\bx) - f_T(\bx)}{\max_{\bz\in B_p(\bx,R)}\|\nabla f_T(\bz)-\nabla f_S(\bz)\|_p},
\label{eq:theorem1}
\end{equation}
we have $f_T(\bx+\bdelta)>f_S(\bx+\bdelta)$.
\end{thm}

\begin{proof}
By the main theorem of calculus, we have
\begin{equation}
  f_S(\bx+\bdelta) = f_S(\bx)+\int_0^1\langle\nabla f_S(\bx+t\bdelta),\bdelta\rangle{\mathrm d}t
  \label{eq:stu_calculus}
\end{equation}
and
\begin{equation}
  f_T(\bx+\bdelta) = f_T(\bx)+\int_0^1\langle\nabla f_T(\bx+t\bdelta),\bdelta\rangle{\mathrm d}t.
  \label{eq:tea_calculus}
\end{equation}
If the perturbation $\bdelta$ is so significant that $f_T(\bx+\bdelta)>f_S(\bx+\bdelta)$, we get
\begin{equation}
  \begin{split}
  &  0<f_S(\bx) -  f_T(\bx) < \\ 
  &  \int_0^1\langle\nabla f_T(\bx+t\bdelta)-\nabla f_S(\bx+t\bdelta),\bdelta\rangle{\mathrm d}t.
  \end{split}
  \label{eq:stu-tea}
\end{equation}
Consider the fact that
\begin{equation}
\begin{split}
  &  \int_0^1 \langle\nabla f_T(\bx+t\bdelta - \nabla f_S(\bx+t\bdelta),\bdelta\rangle{\mathrm d}t\leq \\ 
  &  \|\bdelta\|_q\int_0^1\|\nabla f_T(\bx+t\bdelta) - \nabla f_S(\bx+t\bdelta)\|_p{\mathrm d}t,
\end{split}
\label{eq:holder}
\end{equation}
where holder inequality is applied and q-norm is dual to the p-norm with $\frac{1}{p}+\frac{1}{q}=1$. By combining \Eref{eq:stu-tea} and \Eref{eq:holder}, we have
\begin{equation}
\begin{aligned}
  \|\bdelta\|_q \geq \frac{f_S(\bx) - f_T(\bx)} {\int_0^1\|\nabla f_T(\bx+t\bdelta) - \nabla f_S(\bx+t\bdelta)\|_p{\mathrm d}t},
\end{aligned}
\label{eq:got_delta}
\end{equation}
where the denominator can be further upper bounded using the following inequality
\begin{equation}
\begin{aligned}
   \int_0^1\|\nabla f_T(\bx+t\bdelta - \nabla f_S(\bx+t\bdelta)\|_p{\mathrm d}t \leq \\  
   \max_{\bz\in B_p(\bx,R)}\|\nabla f_T(\bz)-\nabla f_S(\bz)\|_p.
\end{aligned}
\label{eq:upper_denominator}
\end{equation}
The lower bound for the q-norm of $\bdelta$ to break the robust prediction of the student network (\emph{i.e.}, \Eref{eq:what_is_robust}) is therefore
\begin{equation}
   \|\bdelta\|_q\geq \frac{f_S(\bx) - f_T(\bx)}{\max_{\bz\in B_p(\bx,R)}\|\nabla f_T(\bz)-\nabla f_S(\bz)\|_p},
   \label{eq:delta_up_de}
\end{equation}
which completes the proof.
\end{proof}

  \renewcommand{\algorithmicrequire}{\textbf{Input:}}
  \renewcommand{\algorithmicensure}{\textbf{Output:}}
  \begin{algorithm}
  \caption{Robust Student Network Learning}
  \label{alg:classifier}
  \begin{algorithmic}[1]
  \REQUIRE A given neural network $\mathcal N_T$; training dataset $\mathcal X$ with n instances; the corresponding $\bm k$-label set $\mathcal Y$; parameters: $\lambda$, $\beta$, and $\tau$.
  
  \STATE
  Initialize a neural network $\mathcal N_S$, where the number of parameters in $\mathcal N_S$ is significantly fewer than that in $\mathcal N_T$; 
  \STATE
  \textbf{repeat}
  \STATE
  \ \ \ \ Select an instance $\bm x$ and its label $\bm y$ randomly;
  \STATE
  \ \ \ \ Employ the teacher network: $o_T \leftarrow \mathcal N_T(\bm x)$,\\ 
  \STATE
  \ \ \ \ Employ the student network: $o_S \leftarrow \mathcal N_S(\bm x)$;
  \STATE
  \ \ \ \ Calculate the loss function $\mathcal L(\mathcal N_S)$ using \Eref{eq:Gloss};
  \STATE
  \ \ \ \ Update weights in the student network $\mathcal N_S$;
  \STATE
  \textbf{until} reach the limitation of training epoch
  \ENSURE{A robust student network $\mathcal N_S$.}
  \end{algorithmic}
  \end{algorithm}

According to Theorem 1, maximizing the value of $f_S(\bx) - f_T(\bx)$  while minimizing the value of $\max_{\bz\in B_p(\bx,R)}\|\nabla f_T(\bz)-\nabla f_S(\bz)\|_p$, the lower bound over $\bdelta$ will be enlarged, so that the student network is able to tolerate more severe perturbation and become more robust to make confident prediction. Without loss of generality, we take $p=q=2$ in the following discussion.

\subsection{Method}
Based on the analysis above, two new objectives are introduced into teacher-student learning paradigm to achieve a robust student network. To encourage $f_S(\bx)>f_T(\bx)$, we plan to minimize the loss function
\begin{equation}
  \mathcal L_S(\mathcal N_{S}) = \frac{1}{n}\sum_{i=1}^{n} \max(0,\gamma+f_T(\bx^i)-f_S(\bx^i)),
\label{eq:Sloss}
\end{equation}
where $\gamma >0$ is a constant margin. $f_S(\bx)$ is supposed to be greater than $f_T(\bx)+\gamma$, otherwise, there will be a penalty for the student network. 
It is difficult to explicitly calculate the value of $\max_{\bz\in B_2(\bx,R)}\|\nabla f_T(\bz)-\nabla f_S(\bz)\|_2$, due to the existence of the max operation. But by appropriately setting the radius $R$ and considering the sufficiently large training set, the data point in the ball $B_2(\bx,R)$ to reach the maximum value of $\|\nabla f_T - \nabla f_S\|_2$ would often have some closed examples in the training set. Hence, to minimize the value of $\max_{\bz\in B_2(\bx,R)}\|\nabla f_T(\bz)-\nabla f_S(\bz)\|_2$, we proposed to minimize the difference between gradients of student and the teacher networks w.r.t. the training examples as

\begin{equation}
  \mathcal L_G(\mathcal N_{S}) = \frac{1}{n}\sum_{i=1}^{n} \|\nabla \bm\tau(f_S(\bx^i))-\nabla \bm\tau(f_T(\bx^i))\|^2,
\label{eq:Gloss}
\end{equation}
where $\bm\tau$ is the relaxation function explained in \Eref{eq:kdLoss}. In addition, we take the KD loss~\cite{hinton2015distilling} into consideration, the resulting objective function of our robust student network learning algorithm can be written as:
\begin{equation}
  \mathcal L(\mathcal N_{S}) = \mathcal L_{KD}(\mathcal N_{S}) + C_1 \mathcal L_{G}(\mathcal N_{S}) + C_2 \mathcal L_S(\mathcal N_{S}),
\label{eq:final_loss}
\end{equation}
where $C_1$ and $C_2$ are the balanced coefficients of $\mathcal L_{G}(\mathcal N_{S})$ and $\mathcal L_{S}(\mathcal N_{S})$, respectively.

The process of training the student networks can be found in Algorithm 1. After the initialization of the student network, we train the student network according to the proposed algorithm. Next, we explain in detail the calculation of loss. For convenience, we set the batch size as 1, that is, we first select a sample $\{x,y\}$ from the dataset $\mathcal X$ and $\mathcal Y$ as input for forward propagation of the teacher network and the student network. Then we calculate outputs of the two networks $o_T(\bx)$ and $o_S(\bx)$. Combining outputs $o_T(\bx)$ and $o_S(\bx)$ with the corresponding label $y$, we can calculate the first term in \Eref{eq:final_loss} $L_{KD}$ according to \Erefs{eq:kdLoss} and (\ref{eq:tau}). $o_T$ and $o_S$ are both $k$-dimensional vectors, which are the network's prediction scores for $k$ categories. With the help of label $y$, we can get the predicted scores for label, $f_T$ and $f_S$ and calculate the second term in \Eref{eq:final_loss} $L_s$ according to \Eref{eq:Sloss}. In order to get the value of $L_g$, we first calculate the derivative of $f_T$ and $f_S$ with respect to the input sample $x$. Same as back propagation algorithm, we can apply the chain rule to get these results. In experiments, we utilize the automatic derivation tool which is integrated in mainstream deep learning platforms to achieve this process. After getting $f_T$ and $f_S$, the loss can be easily calculated using \Eref{eq:final_loss}. Finally, the weights in the student network are updated by the gradients obtained by the back-propagation algorithm.

The delta $\bdelta$ in \Eref{eq:theorem1} is the noise in an image $x$. The noise can come from various sources. Some are physical, linked to the nature of light and to optical artifacts, and some others are created during the conversion from electrical signal to digital data. As noise degrades the quality of an image, the performance of neural networks in image classification task could be seriously influenced. The proposed robust student network aims to handle unexpected noises in images and to preserve consistent decisions with or without noises (see \Erefs{eq:constrain1} and (\ref{eq:what_is_robust}). In the literature, the overall noise produced by a digital sensor is usually considered as a stationary white additive Gaussian noise~\cite{julliand2015image}. We report the robustness of the learned networks against Gaussian noise in experiment. In addition, we also evaluate the performance of the learned networks against combinations of different types of noise on training and test sets, since it is difficult to know what types of noise could be before the test stage.

\begin{figure*}
\setlength{\abovecaptionskip}{0.2cm}
\setlength{\belowcaptionskip}{0.2cm}
\centering
\begin{tabular}{ccc}
\centering
\includegraphics[width=0.305\linewidth]{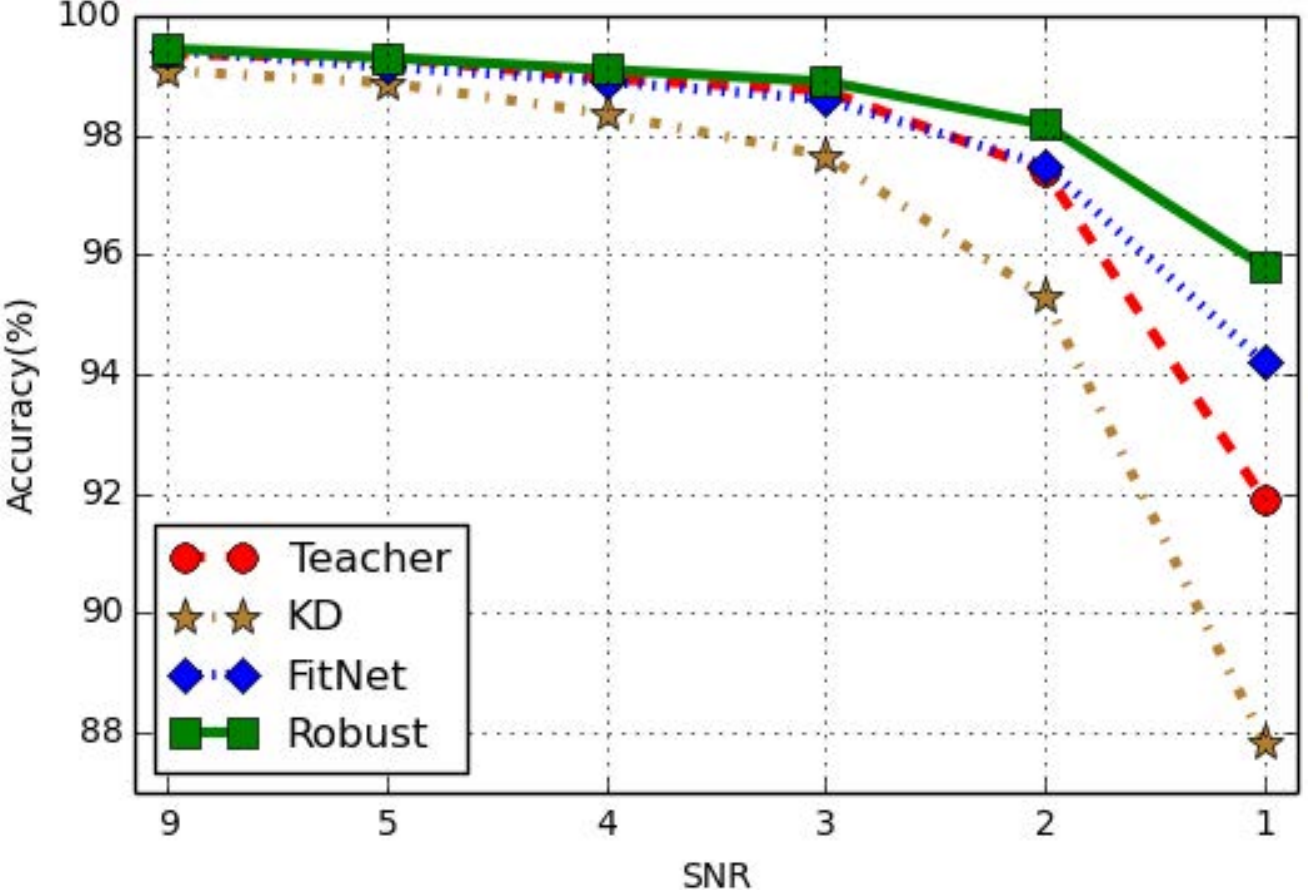}\ \ \ \ & \includegraphics[width=0.30 \linewidth]{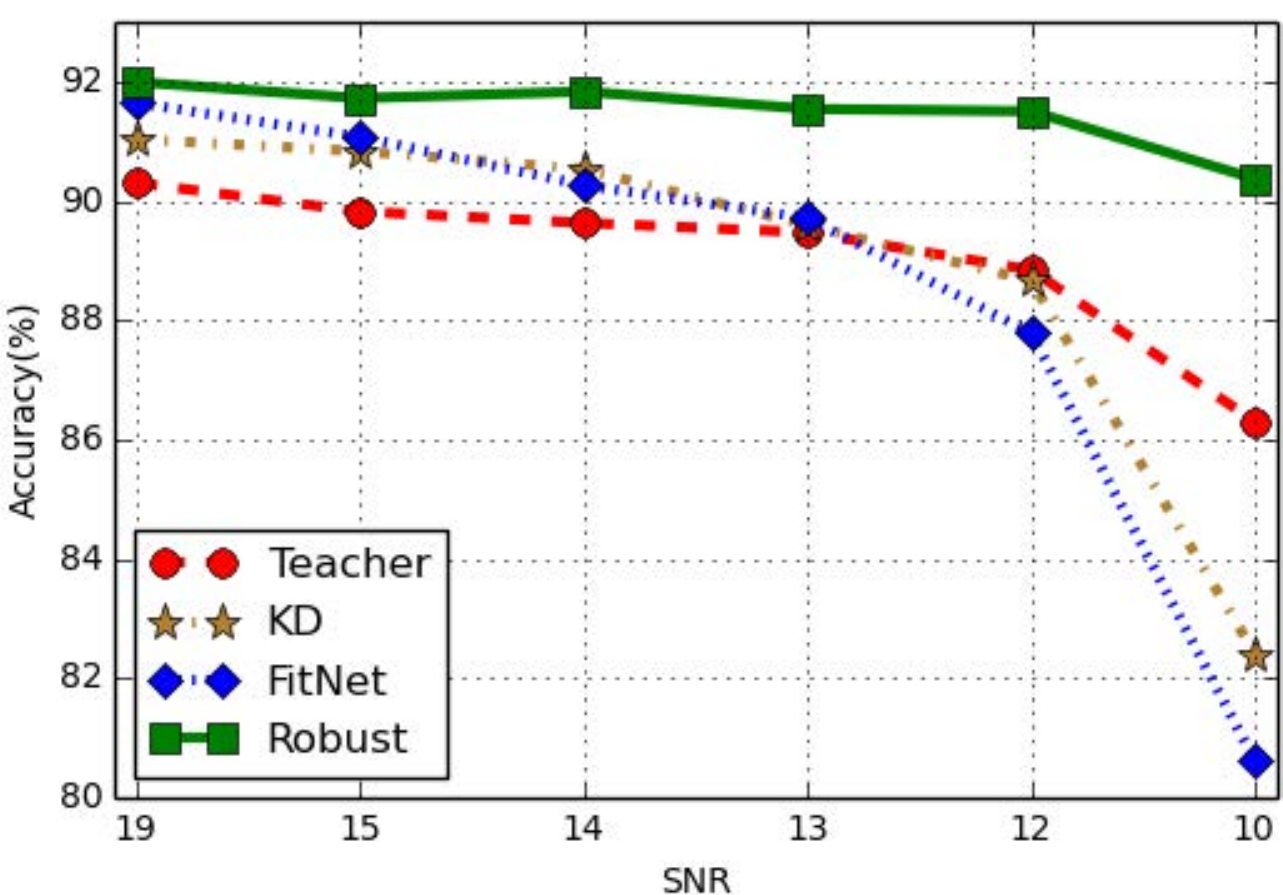}\ \ \ \  & \includegraphics[width=0.30 \linewidth]{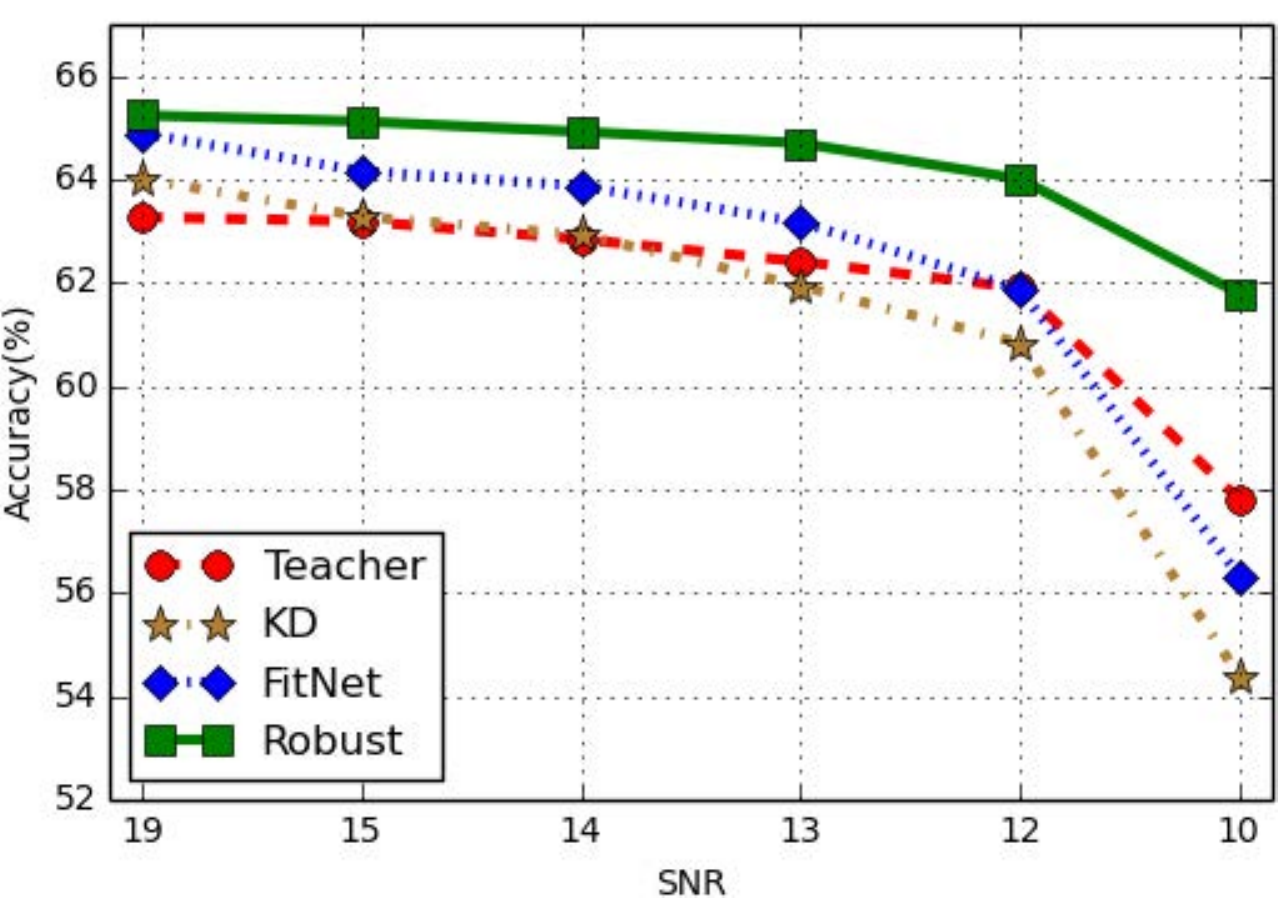} \\
(a) MNIST & (b) CIFAR-10 & (c) CIFAR-100 \\
\end{tabular}
\caption{Accuracies obtained by different networks trained on three datasets and under various values of SNR.}
\label{fig2}
\end{figure*}

\section{Experiments}
In this section, we experimentally investigate the effectiveness of the proposed robust student network learning algorithm. The learned student network is compared with the original teacher network, and the student networks learned through KD~\cite{hinton2015distilling} and Fitnet~\cite{romero2014fitnets}. The experiments are on three benchmark datasets: MNIST~\cite{lecun1998gradient}, CIFAR-10~\cite{krizhevsky2009learning}, and CIFAR-100~\cite{krizhevsky2009learning}.

\subsection{Dataset and Settings}
MNIST~\cite{lecun1998gradient} is a handwritten digit dataset (from 0 to 9) composed of $28\times 28$ greyscale images from ten categories. The whole dataset of 70,000 images is split into 60,000 and 10,000 images for training and test, respectively. Following the setting in~\cite{romero2014fitnets}, we trained a teacher network of maxout convolutional layers reported in~\cite{goodfellow2013maxout}, which contains 3 convolutional maxout layers and a fully-connected layer with 48-48-24-10 units, respectively. After that, we design the student network which contains 6 convolutional maxout layers and a fully-connected layer, which is twice as deep as the teacher network but with roughly 8\% of the parameters. As reported in \Tref{tab_arch_4}, the architectures of the teacher and student network were shown in detail in the first two columns.

CIFAR-10~\cite{krizhevsky2009learning} is a dataset that consists of $32\times 32$ RGB color images draw from 10 categories. There are 60,000 images in CIFAR-10 dataset which are split into 50,000 training and 10,000 testing images. According to~\cite{goodfellow2013maxout} and~\cite{romero2014fitnets}, we preprocessed the data using global contrast normalization (GCA) and ZCA whitening, and augmented the training data via random flipping. We followed the architecture used in Maxout~\cite{goodfellow2013maxout} and FitNet~\cite{romero2014fitnets} to train a teacher network with three maxout convolutional layers of 96-192-192 units. For fair comparison, we designed a student network with a structure similar to FitNet which has 17 maxout convolutional layers followed by a maxout fully-connected layer and a top softmax layer, and we also investigate KD method with the same architecture. The detailed architecture of teacher was shown in the `Teacher(CIFAR-10)' column of \Tref{tab_arch_4}, and that of student was shown as `Student 4' column.

CIFAR-100 dataset~\cite{krizhevsky2009learning} has images of the same size and format as those in CIFAR-10, except that it has 100 categories with only one tenth as labeled images per category. More categories and fewer labeled examples per category indicates that classification task on CIFAR-100 is more challenging than that on CIFAR-10. We preprocess images in CIFAR-100 using the same methods for CIFAR-10 and the teacher network and the student network share the same hyper-parameters with those on the CIFAR-10 dataset. Besides, the architecture of teacher is also same as that used for CIFAR-10, except that the number of units in the last softmax layer was changed to 100 to adapt to the number of categories. 

The hyper-parameters are tuned by minimizing the error on a validation set consisting of the last 10,000 training examples on each dataset. Following the setting in FitNet~\cite{romero2014fitnets}, we set batch size as 128, max training epoch as 500, learning rate as 0.17 for linear layers and 0.0085 for convolutional layers, and momentum as 0.35. According to the hint layer proposed in FitNet~\cite{romero2014fitnets}, we pre-trained a classifier using the features of the teacher network's middle layer, and then we apply the classifier with the student network features.

\begin{figure*}
\setlength{\abovecaptionskip}{0.2cm}
\setlength{\belowcaptionskip}{0.2cm}
\centering

\begin{tabular}{cccc}
\centering
\small
\includegraphics[width=0.18\linewidth]{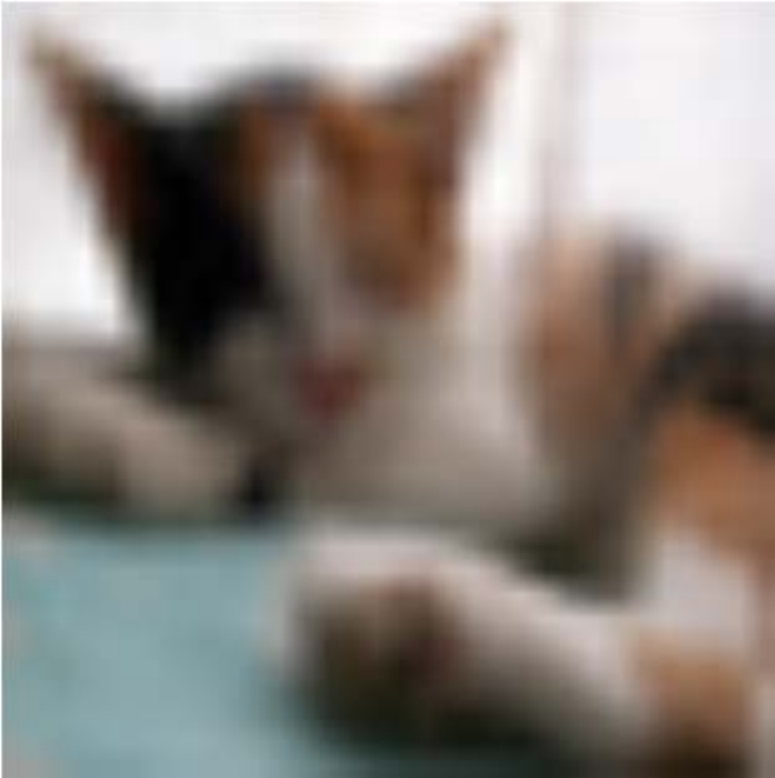} & \includegraphics[width=0.18\linewidth]{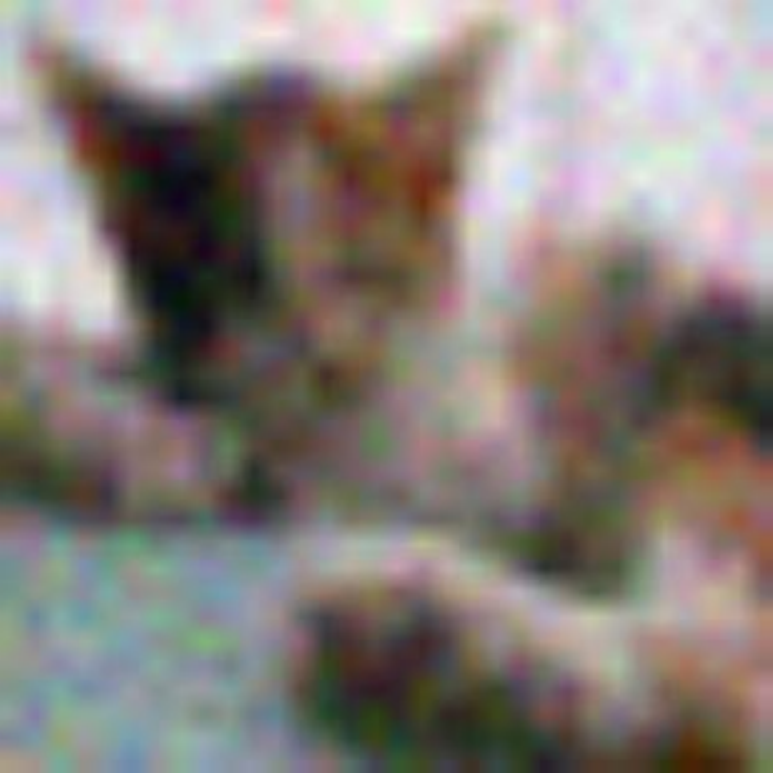} &\includegraphics[width=0.18\linewidth]{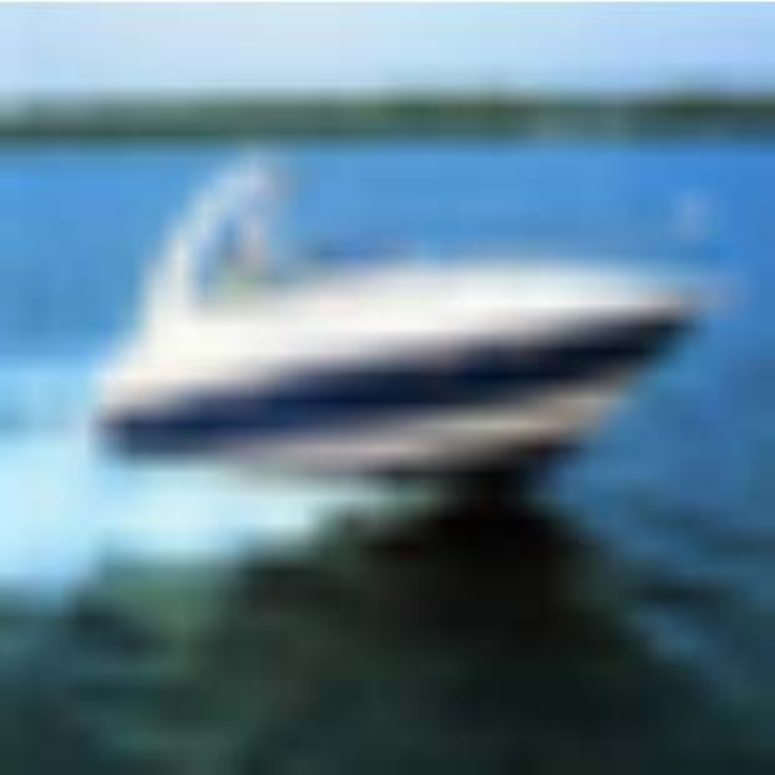} & \includegraphics[width=0.18\linewidth]{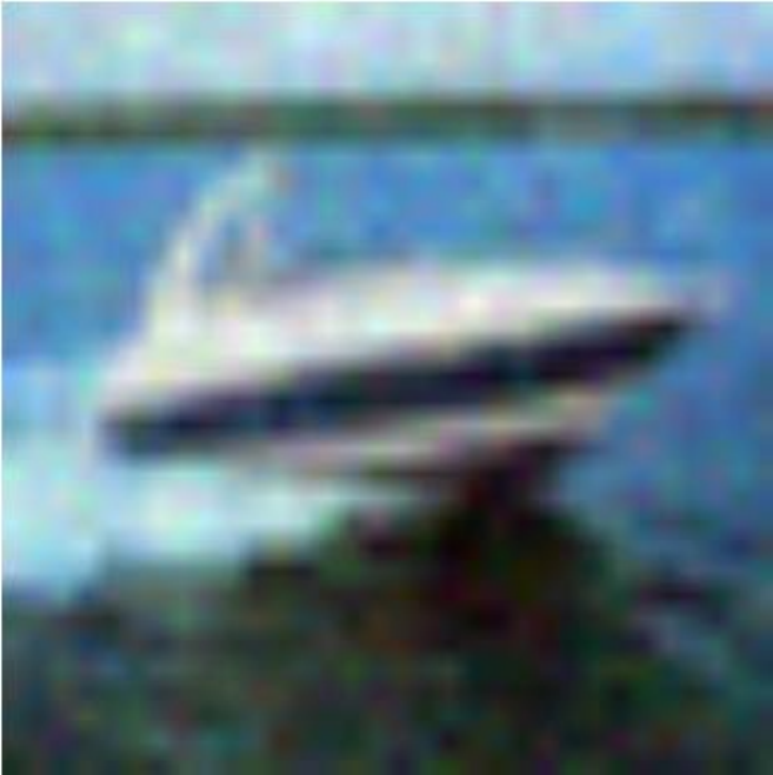} \\
\includegraphics[width=0.23\linewidth]{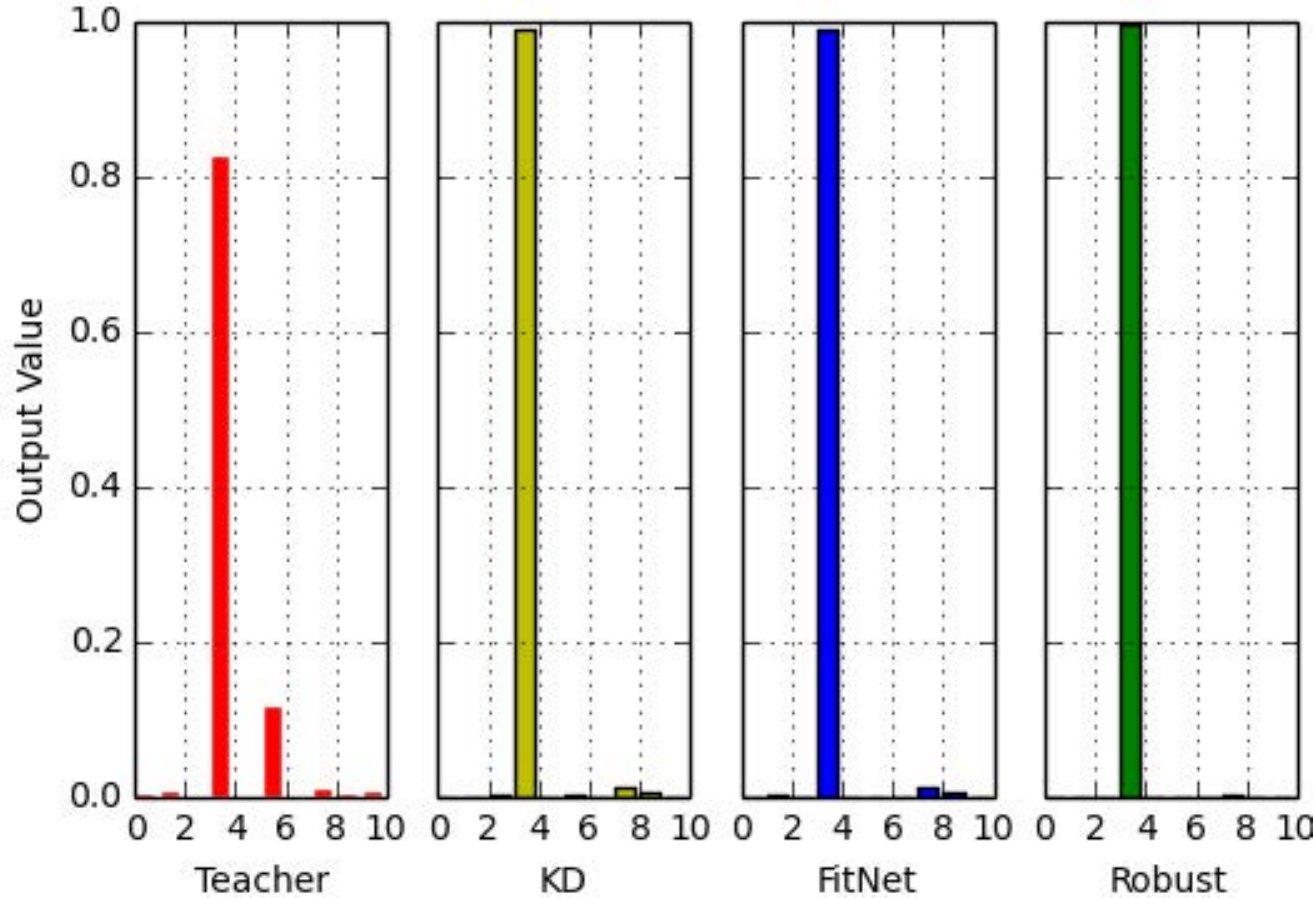} & \includegraphics[width=0.23\linewidth]{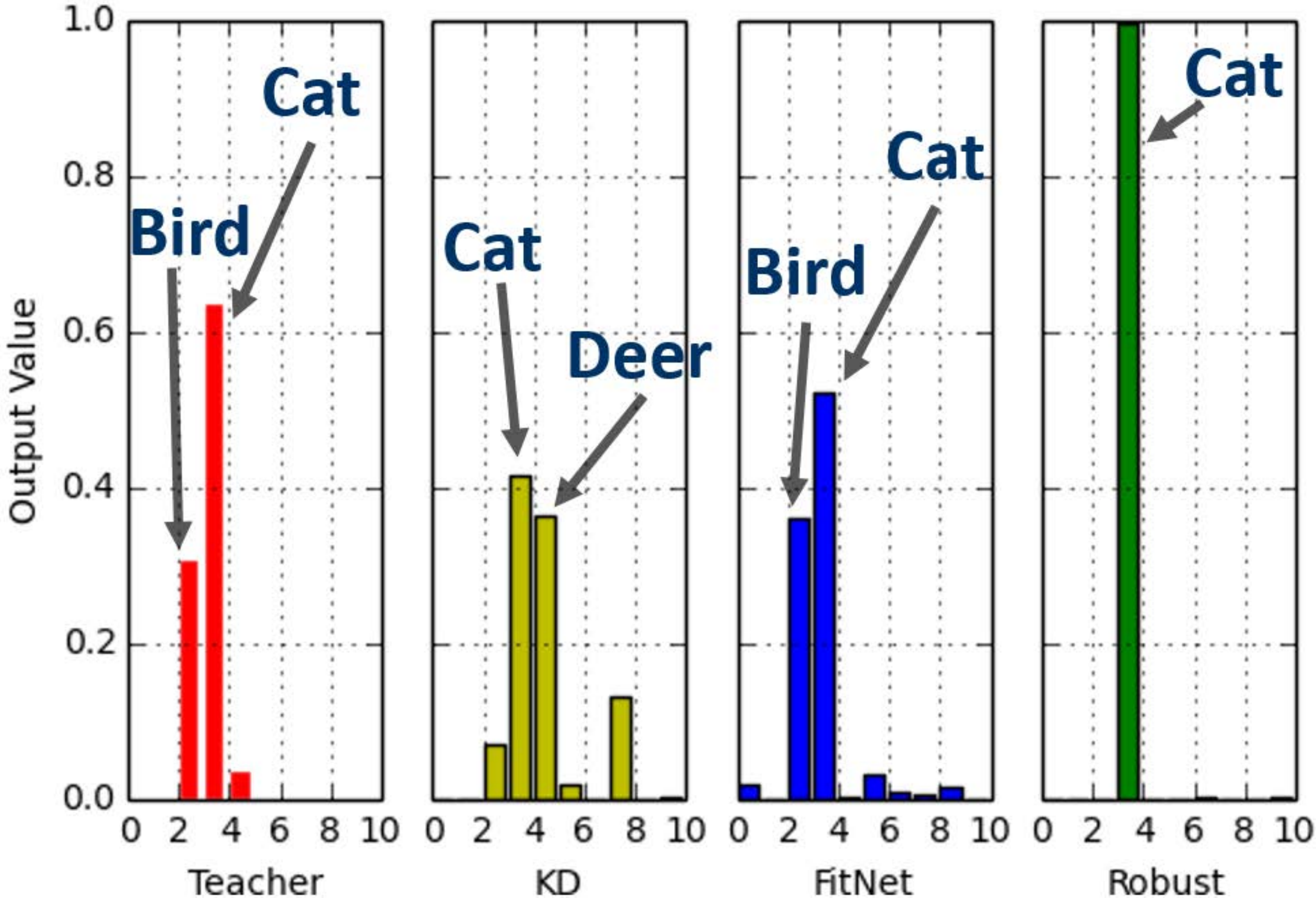} &\includegraphics[width=0.23\linewidth]{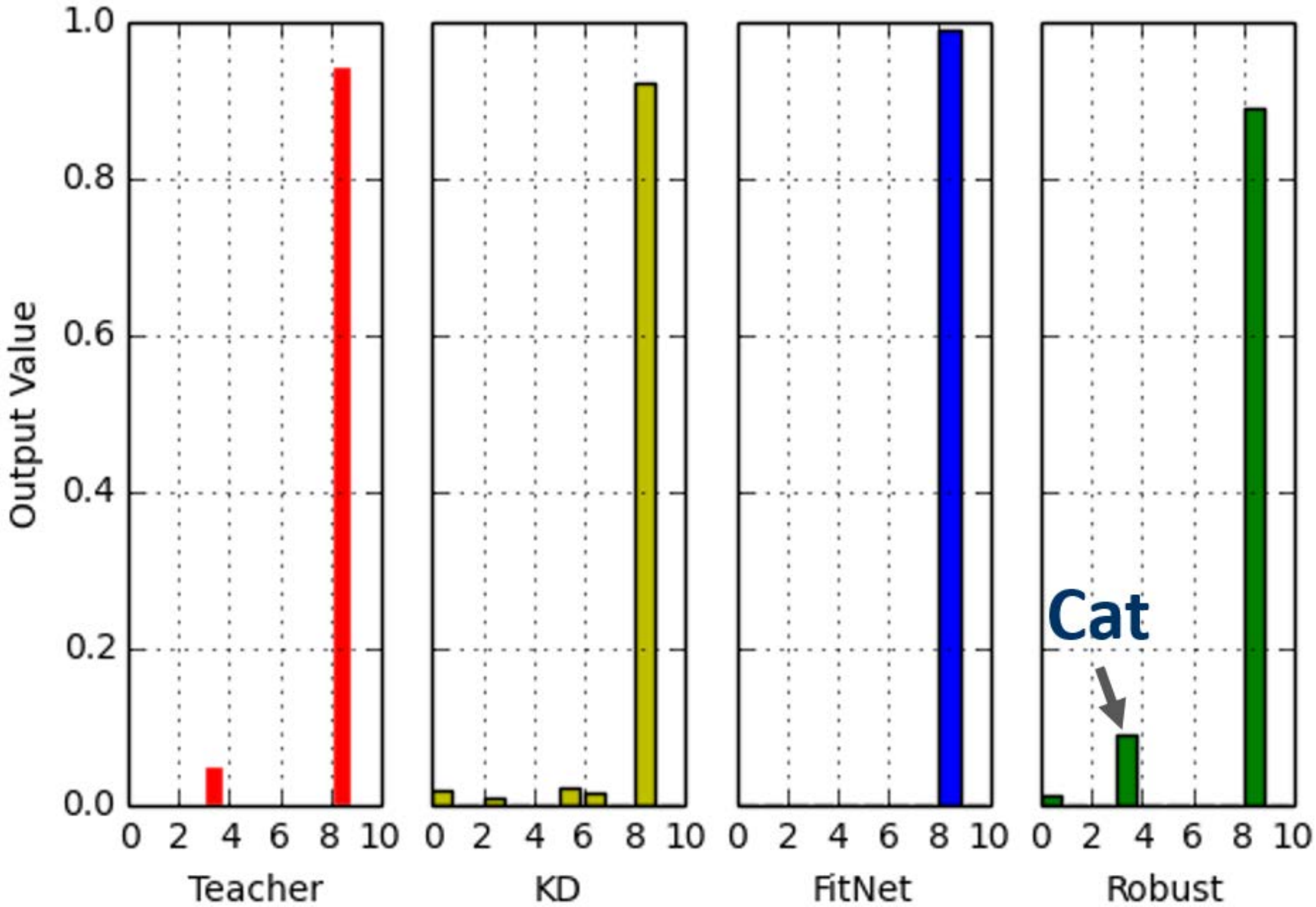} & \includegraphics[width=0.23\linewidth]{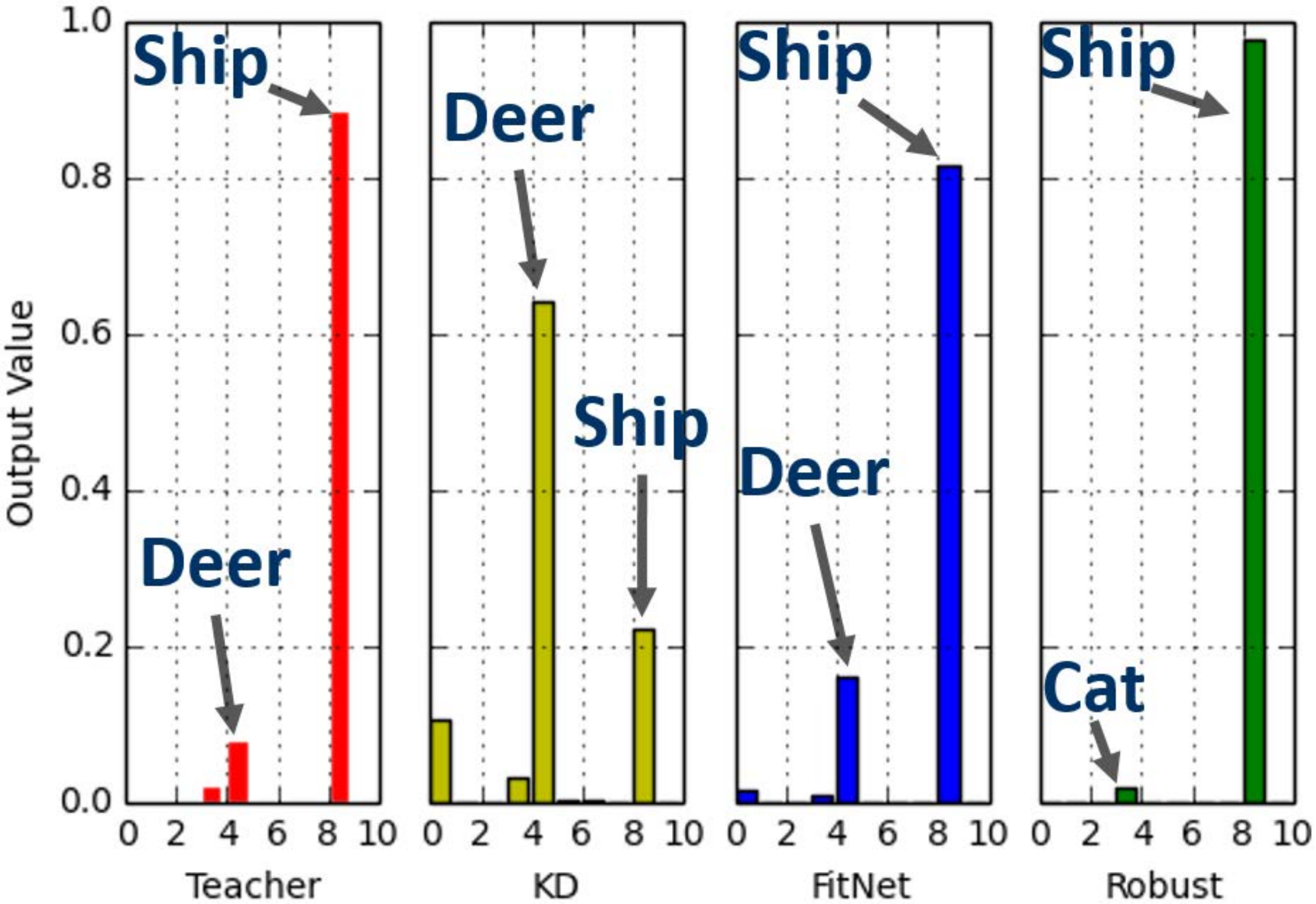} \\
(a) Clean Cat & (b) Cat with SNR=10 & (c) Clean Ship & (d) Ship with SNR=10 \\
\end{tabular}
\caption{Example images (the top line) and their corresponding prediction scores by different networks (the bottom line). (a) and (c) are pure images, while (b) and (d) are disturbed images.}
\label{fig3}
\end{figure*}

\subsection{Robustness of Student Networks}

We evaluated the robustness of student networks learned through different algorithms under different intensities of perturbation. Since it is difficult and impossible to know what test data can be in practice, the augmentation of training data with certain noise cannot be very helpful to resist the perturbation. Hence, we trained all networks using clean training set, and introduced White Gaussian Noise (WGN) into test data as the perturbation. The intensity of the introduced noise was measured in terms of Signal-to-Noise Ratio (SNR). We trained the proposed algorithm and compared it with the teacher network~\cite{goodfellow2013maxout}, and student networks from KD network~\cite{hinton2015distilling} and FitNet~\cite{romero2014fitnets} methods.

In \Fref{fig2}, we investigated the accuracy of these networks on three datasets with different SNR values. As the classification task on MNIST is relatively easier, lower SNRs were chosen from 9 to 1. Lower SNR value indicates more perturbations are added. It can be found from \Fref{fig2}(a) that the accuracy of the proposed robust student network is superior to other three networks nearly under all SNR values. When SNR equals to 2, two student networks from KD and FitNet perform even worse than the original teacher network. But our proposed algorithm achieves an obviously leading 98.17\% accuracy. When SNR was down to 1, the accuracy drops of the teacher network, and the student network from KD and FitNet are serious, up to 5.65\%, 7.25\%, and 3.23\%, respectively. In contrast, the accuracy of our robust student network only drops 2.23\%. Our method achieves better performance and shows more robustness when there was perturbation in the input.

A similar phenomenon can be observed in Figures \ref{fig2} (b) and (c) on the CIFAR-10 and CIFAR-100 datasets. With the decrease of SNR, the accuracy of KD network and FitNet dropped faster than that of the teacher network, especially during the period when SNR drops from 12 to 10. Given the significant reduction in network complexity, the capacity of the student network can be seriously weakened and the student network would be more vulnerable to perturbations on data if there is no appropriate response action. However, the student network learned from the proposed algorithm can be robust to serious perturbations. 

In \Fref{fig3}, we reported the predicted scores of example images by different methods on the CIFAR-10 dataset. The clean image without noise looks fuzzy, since the images from CIFAR-10 dataset only a resolution of $32\times 32$. In the first column of \Fref{fig3}, all student networks can confidently predict the ground-truth class `cat' of the image. However, given the same image added with SNR=10 noise in \Fref{fig3} (b), though student networks from KD and FitNet methods reluctantly made the correct prediction, KD also thought the image is similar to `deer', and FitNet trusted 'bird' as the prediction with a higher confidence level. In contrast, our robust student network confidently insisted on its correct prediction even the quality of image has been seriously influenced by the perturbation. In addition, given the `ship' image, the teacher network can stand against the perturbation, due to its strong capability coming from the complicated network structure. The KD method mistook it as an `deer' image, while FitNet assigned higher score to label `deer' for this `ship' image. By encouraging more confident predictions with the help of the teacher network during the training stage, we derive the robust student network that can not only keep the highest prediction score on the `ship' label, but also suppress the predictions on wrong categories (see label `cat' in Figures \ref{fig3}(c) and (d)).

\begin{table*}[htb]
\setlength{\abovecaptionskip}{0.2cm}
\setlength{\belowcaptionskip}{0.2cm}
\renewcommand\arraystretch{1}
\centering
\small
\caption{Performance comparison on different training and test sets. Dataset was split as training/test. `C' represents clean data, `G' represents data with Gaussian noise, and `P' represents data with Poisson noise.}
\begin{tabular}{|c|*{7}{p{1.15cm}<{\centering}|}}
\hline
\textbf{Network}  &  \textbf{C/C}  &  \textbf{C/G}  &  \textbf{C/P}  &  \textbf{G/G}  &  \textbf{G/P}  &  \textbf{P/P} &  \textbf{P/G}   \\
\hline
Teacher~\cite{goodfellow2013maxout}  & 90.25\%  & 86.30\%  & 86.60\%  & 89.02\%  & 87.36\%  & 89.11\%  & 86.06\% \\
\hline
KD~\cite{hinton2015distilling}     & 91.07\%  & 80.61\%  & 80.86\%  & 90.48\%  & 82.14\%  & 90.63\%  &  82.27\% \\
\hline
FitNet~\cite{romero2014fitnets}  & 91.64\%  & 82.41\%  & 82.43\%  & \textbf{90.86}\%  & 86.11\%  & \textbf{91.10}\%  & 84.02\%  \\
\hline
\hline
Robust (proposed)  & \textbf{91.93}\%  & \textbf{90.37}\%  & \textbf{90.50}\%  & 90.37\%  &  \textbf{90.50}\%  & 90.50\%  & \textbf{90.37}\%  \\
\hline
\end{tabular}
\label{tab_CGP}
\end{table*}

\begin{table*}[htb]
\setlength{\abovecaptionskip}{0.2cm}
\setlength{\belowcaptionskip}{0.2cm}
\renewcommand\arraystretch{1}
\centering
\small
\caption{Performance comparison on different block sizes.}
\begin{tabular}[b]{|c|c|c|c|c|c|c|}
\hline
\textbf{Dataset}  &  \textbf{Algorithm}  & \textbf{2$\times$2 block}  &  \textbf{4$\times$4 block}  &  \textbf{6$\times$6 block}  &  \textbf{8$\times$8 block}   &  \textbf{10$\times$10 block}   \\
\hline
\hline
\multirow{4}*{\textbf{CIFAR10}} 
& Teacher~\cite{goodfellow2013maxout}  & 89.57\%  & 89.03\%  & 87.69\%  & 85.65\%  & 81.94\%  \\
\cline{2-7}
& KD~\cite{hinton2015distilling}  & 90.53\%  & 89.53\%  & 87.23\%  & 83.37\%  & 77.58\%  \\
\cline{2-7}
& FitNet~\cite{romero2014fitnets}  & 91.08\%  & 89.94\%  & 87.40\%  & 84.62\%  & 79.71\%  \\
\cline{2-7}
& Robust(proposed)  & \textbf{91.25}\%  & \textbf{90.79}\%  & \textbf{88.34}\%  &  \textbf{85.92}\%  & \textbf{82.15}\%  \\
\hline
\hline
\multirow{4}*{\textbf{CIFAR100}} & Teacher~\cite{goodfellow2013maxout}  & 63.07\%  & 61.92\% & 59.95\%  & 57.41\%  &  \textbf{55.63}\%   \\
\cline{2-7}
& KD~\cite{hinton2015distilling}  & 63.48\%  & 61.84\% & 58.73\%  & 54.38\%  &  51.67\%   \\
\cline{2-7}
& FitNet~\cite{romero2014fitnets}  & 64.11\%  & 62.62\% & 59.65\%  & 55.51\%  &  52.40\%   \\
\cline{2-7}
& Robust(proposed)  & \textbf{64.83}\%  & \textbf{63.16}\% & \textbf{60.68}\%  & \textbf{57.49}\%  &  55.11\%   \\
\hline
\end{tabular}
\label{tab_block}
\end{table*}

\subsection{Comparison under Different Perturbation}

A neural network might handle noisy test data, if similar noise also exists in the training set. However, in practice, it is difficult to guarantee the test data to have the same kind of perturbation as the training data. We next proceed to evaluate the performance of different methods under different combinations of noisy training and test sets. The accuracies in different settings are presented in \Tref{tab_CGP}. The first line in \Tref{tab_CGP} is a description of the experiment settings. The first capital letter indicates the noise type introduced to training set, and the second letter indicates that of test set. It should be noted that the results of the Robust network listed in this table are all trained under the clean data set, but tested under the corresponding type of noise indicated by the first line. If both training and test data are clean, all networks can achieve more than 90\% accuracy, as shown in the first column of \Tref{tab_CGP}. If networks are trained on the clean data and tested on the data with Gaussian noise, teacher networks and student networks from KD and FitNet will be seriously influenced and can only achieve less then 86\% accuracy. However, the proposed robust student can still own more than 90\% accuracy. A similar phenomenon can be observed when the networks are trained with Gaussian noise but tested with Poisson noise and reverse, as shown in the fourth and last columns of \Tref{tab_CGP}, respectively. If both training and test data are polluted with the same type of Gaussian noise, all networks would try to fit the noisy data as far as possible and receive only slight performance drop. But this rigorous constraint over training and test data cannot always be satisfied in real-world applications. It shows that adding limited kinds of noise to the training set is difficult to improve the robustness of neural network when facing various unexpected kinds of noises existing in practical applications.

Moreover, \Tref{tab_CGP} shows that the teacher Network's accuracy is worse than that of Robust Network when training and test sets are both with Gaussian noise. The distributions of training data and test data would not be significantly different, if they are both polluted by the Gaussian noise. Hence general teacher networks and student networks can well fit the noisy data and receive reasonable accuracy. But the student network achieves some performance improvement, because of its deeper architecture than that of teacher networks. The depth encourages the re-use of features, and leads to more abstract and invariant representations at higher layers. The proposed robust student network can successfully train a deeper network by exploiting information from the teacher network.

\subsection{Complex Perturbation}
In real-world applications, noise is not the only perturbation that may be encountered. Some more complex perturbation also challenges the robustness of neural networks. In this section, we investigate the robustness of the student network obtained by our proposed method on two more complex perturbations, i.e. image occlusion and domain adoption. 

\subsubsection{Image Occlusion}
Considering the target object in the real environment is often blocked, and such perturbation often results in the loss of information in a continuous area, the performance of neural networks will be influenced more seriously by image occlusion. 
In order to investigate the robustness of our method under this disturbance, we take image occlusion as a more complex perturbation. To simulate the occlusion in real-world applications, we randomly select a small rectangular area in an image, and set pixels covered by the rectangle as zeros. Five different block sizes, \emph{i.e.} 2$\times$2, 4$\times$4, 6$\times$6, 8$\times$8 and 10$\times$10 are used in experiments. We implemented this experiment on the CIFAR-10 dataset. The results are shown in \Tref{tab_block}. Given 4$\times$4 blocks, teacher, KD, FitNet, and Robust network respectively have accuracies of 89.03\%, 89.53\%, 89.94\%, and 90.79\%. Given 8$\times$8 blocks, the corresponding accuracies are 85.65\%, 83.37\%, 84.62\%, and 85.92\%, respectively. According to these results, larger blocks indicate more serious perturbations of images, which will degrade the performance of neural networks. However, the student network obtained by the proposed method stably stays ahead, because of its robustness. 

\begin{table}[hb]
\setlength{\abovecaptionskip}{0.2cm}
\setlength{\belowcaptionskip}{0.2cm}
\renewcommand\arraystretch{1}
\centering
\small
\caption{Domain adaptation results.}
\begin{tabular}{|c|c|c|c|c|}
\hline
\multirow{2}*{\textbf{Algorithm}}  &  \multicolumn{2}{c|}{\textbf{MNIST2USPS}}  &  \multicolumn{2}{c|}{\textbf{USPS2MNIST}} \\
\cline{2-5}
&  MNIST  &  USPS  &  USPS  &  MNIST  \\ 
\hline
\hline
Teacher~\cite{goodfellow2013maxout}  & 99.45\%  & --  & 96.41\%  &  86.88\% \\
\hline
KD~\cite{hinton2015distilling}  & 99.35\%  & 93.25\%  & 96.26\%  &  82.74\% \\
\hline
FitNet~\cite{romero2014fitnets}  & 99.49\% & 94.12\%  & 96.56\%  &  87.23\% \\
\hline
\hline
Robust(proposed)  & 99.55\%  &  \textbf{95.02}\%  &  96.71\%  &  \textbf{89.14}\% \\
\hline
\end{tabular}
\label{tab_domain}
\end{table}

\begin{table*}[ht]
\renewcommand\arraystretch{1}
\centering
\small
\setlength{\abovecaptionskip}{0.2cm}
\setlength{\belowcaptionskip}{0.2cm}
\caption{Classification accuracies of different networks on CIFAR-10 and CIFAR-100 datasets.}
\begin{tabular}{|c|c|c|c|c|}
\hline
\textbf{Algorithm}  &  \textbf{$\#$\textbf{params}}  &  \textbf{$\#$\textbf{layers}}  &  \textbf{\textbf{CIFAR-10}}  &  \textbf{\textbf{CIFAR-100}} \\
\hline
\hline
\multicolumn{5}{|c|}{\emph{Student-teacher learning paradigm}} \\
\hline
\hline
Teacher~\cite{goodfellow2013maxout}  &  $\sim$ 9M  &  5  &  90.25\%  &  63.49\%  \\
\hline
Knowledge Distillation~\cite{hinton2015distilling}   &  $\sim$ 2.5M  & 19  &  91.07\%  &  64.13\%  \\
\hline
FitNet~\cite{romero2014fitnets}  &  $\sim$ 2.5M  &  19  &  91.64\%  &  64.86\%  \\
\hline
Robust learning(proposed)  &  $\sim$ 2.5M  &  19  &  \textbf{91.93\%}  &  \textbf{65.28}\%  \\
\hline
\hline
\multicolumn{5}{|c|}{\emph{State-of-the-art-methods}} \\
\hline
\hline
\multicolumn{3}{|c|}{Maxout Network~\cite{goodfellow2013maxout}}  &  90.62\%  &  61.43\% \\
\hline
\multicolumn{3}{|c|}{Network in Network~\cite{lin2013network}}  &  91.20\%  &  64.32\% \\
\hline
\multicolumn{3}{|c|}{Deeply-Supervised Networks~\cite{lee2015deeply} }  &  \textbf{91.78\%}  &  \textbf{65.43}\% \\
\hline
\end{tabular}
\label{tab_cifar}
\end{table*}

\begin{table*}[htb]
\setlength{\abovecaptionskip}{0.2cm}
\setlength{\belowcaptionskip}{0.2cm}
\renewcommand\arraystretch{1}
\centering
\small
\caption{10-Class classification accuracies of different networks on CIFAR-10}
\begin{tabular}{|c|*{10}{p{1.cm}<{\centering}|}}
\hline
\textbf{Algorithm} &  \textbf{plane}  &  \textbf{car}  &  \textbf{bird}  &  \textbf{cat}  &  \textbf{deer}  &  \textbf{dog}  &  \textbf{frog}  &  \textbf{horse}  &  \textbf{ship}  &  \textbf{truck}   \\
\hline
\hline
Teacher~\cite{goodfellow2013maxout}  & 90.1\%  &  93.8\%  &  86.0\%  &  74.6\%  &  93.5\%  &  86.2\%  &  95.2\%  &  92.6\%  &  95.3\%  &  95.2\%   \\
\hline
KD~\cite{hinton2015distilling}  & 90.0\%  &  95.2\%  &  83.2\%  &  84.4\%  &  93.2\%  &  87.1\%  &  95.0\%  &  91.6\%  &  97.3\%  &  93.7\%  \\
\hline
FitNet~\cite{romero2014fitnets}  & 90.7\%  &  97.6\%  &  91.0\%  &  82.7\%  &  93.8\%  &  86.2\%  &  92.7\%  &  93.6\%  &  94.6\%  &  93.5\%   \\
\hline
\hline
Robust(proposed)  & 91.0\%  &  97.0\%  &  90.3\%  &  83.6\%  &  92.4\%  &  87.2\%  &  95.4\%  &  93.2\%  &  95.1\%  &  94.1\%   \\
\hline
\end{tabular}
\label{tab_cifar_class}
\end{table*}

\begin{table}[hb]
\renewcommand\arraystretch{1}
\centering
\small
\setlength{\abovecaptionskip}{0.2cm}
\setlength{\belowcaptionskip}{0.2cm}
\caption{Classification accuracies on `MNIST' dataset.}
\begin{tabular}{|c|c|c|}
\hline
\textbf{Algorithm}  &  \textbf{$\#$\textbf{params}}  &  \textbf{\textbf{Misclass}} \\
\hline
\hline
\multicolumn{3}{|c|}{\emph{Student-teacher learning paradigm}} \\
\hline
\hline
Teacher  &  $\sim$ 361K  &  0.55\%  \\
\hline
Standard back-propagation  &  $\sim$ 30K  &  1.90\%  \\
\hline
Knowledge Distillation~\cite{hinton2015distilling}   &  $\sim$ 30K  &  0.65\%  \\
\hline
FitNet~\cite{romero2014fitnets}  &  $\sim$ 30K  &  0.51\%  \\
\hline
Robust(proposed)  &  $\sim$ 30K  &  \textbf{0.45\%}  \\
\hline
\hline
\multicolumn{3}{|c|}{\emph{State-of-the-art-methods}} \\
\hline
\hline
\multicolumn{2}{|c|}{Maxout Network~\cite{goodfellow2013maxout}}  &  0.45\% \\
\hline
\multicolumn{2}{|c|}{Network in Network~\cite{lin2013network}}  &  0.47\% \\
\hline
\multicolumn{2}{|c|}{Deeply-Supervised Networks~\cite{lee2015deeply}}  &  \textbf{0.39\%} \\
\hline
\end{tabular}
\label{tab_mnist}
\end{table}

\subsubsection{Domain Adaptation}
In practical applications, not only unexpected noise and occlusion, but also the unexpected distribution shift could challenge the robustness of neural networks. It is also an important indicator to evaluate the adaptability of this algorithm in the task of domain adaptation. 

In this experiment, we took the USPS dataset obtained from the scanning of handwritten digits from envelopes by the U.S. Postal Service. The images in this dataset are all $16\times16$ grayscale images and the values have been normalized. The whole dataset has 9,298 handwritten numeric images, of which 7,291 are for training, and the remaining 2,007 are for validation. Similar to MNIST, the USPS dataset has 10 categories, but it has different numbers of samples per category. In addition, considering the picture size in the MNIST dataset is $28\times28$, for convenience, we pad the images in the USPS dataset to the same size. Moreover, we preprocess USPS datasets in the same way as MNIST.

In this section, we train student networks on the MNIST dataset, and test them on USPS dataset. Similarly, we train networks on the USPS dataset and test them on MNIST. The results are shown in the \Tref{tab_domain}. The first two columns show the result of adapting MNIST to USPS, and the performance of adapting USPS to MNIST was reported in the last two columns of this table. According to the results, the proposed algorithm achieves an accuracy of 95.02\%, while the comparison methods KD and FitNet only get 93.25\% and 94.12\%, respectively. This demonstrates that the proposed robust student network can preserve its robustness advantages over comparison methods, when faced with more complex perturbation of data in domain adaptation task. The similar phenomenon can be observed in the results of `USPS to MNIST'. With the similar accuracy on USPS dataset, the Robust Network outperforms networks obtained by the other algorithms. Moreover, the results tested on USPS dataset while trained on MNIST dataset are much better than those tested on MNIST and trained on USPS. This is because that the number of pictures in the MNIST dataset is much larger than that of USPS. The networks trained by MNIST dataset could extract more useful information from a larger amount of data, and thus has better generalization capabilities. 

\begin{table*}[htb]
\setlength{\abovecaptionskip}{0.2cm}
\setlength{\belowcaptionskip}{0.2cm}
\renewcommand\arraystretch{1}
\centering
\small
\caption{The performance of the proposed method on student networks with various architectures.}
\begin{tabular}{|c|c|c|c|c|c|c|c|}
\hline
\textbf{Network}  &  \textbf{$\#$\textbf{layers}}  &  \textbf{$\#$\textbf{params}}  &  \textbf{$\#$\textbf{mult}}  &  \textbf{\textbf{Speed-up Ratio}}  &  \textbf{\textbf{Compression Ratio}}  &  \textbf{\textbf{FitNet}}  &  \textbf{\textbf{Robust}}\\
\hline
\hline
Teacher  &  5  &  $\sim$ 9M  &  $\sim$ 725M  &  $\times 1$  &  $\times 1$  &  \multicolumn{2}{c|}{ 90.25\%}    \\
\hline
Student 1  &  11  &  $\sim$ 250K  &  $\sim$ 30M  &  $\times \textbf{13.17}$  &  $\times \textbf{36}$  &  89.07\%  &  89.62\%    \\
\hline
Student 2  &  11  &  $\sim$ 862K  &  $\sim$ 108M  &  $\times 4.56$  &  $\times 10.44$  &  91.02\%  &  91.37\%    \\
\hline
Student 3  &  13  &  $\sim$ 1.6M  &  $\sim$ 392M  &  $\times 1.40$  &  $\times 5.62$  &  91.16\%  &  91.50\%    \\
\hline
Student 4  &  19  &  $\sim$ 2.5M  &  $\sim$ 382M  &  $\times 1.58$  &  $\times 3.60$  &  91.64\%  &  \textbf{91.93}\%    \\
\hline
\end{tabular}
\label{tab_students_result}
\end{table*} 

\begin{table*}[htb]
\renewcommand\arraystretch{1}
\centering
\small
\setlength{\abovecaptionskip}{0.2cm}
\setlength{\belowcaptionskip}{0.2cm}
\caption{Model architecture for datasets.}
\begin{tabular}{|c|c|c|c|c|c|c|}
\hline
Teacher(MNIST)  &  Student(MNIST)  &  Teacher(CIFAR)  &   Student 1  &  Student 2  &  Student 3  &  Student 4  \\
\hline
\hline
conv 3x3x48  &  conv 3x3x16  &  conv 3x3x96  &  conv 3x3x16  &  conv 3x3x16  &  conv 3x3x32  &  conv 3x3x32  \\
pool 4x4     &  conv 3x3x16  &  pool 4x4     &  conv 3x3x16  &  conv 3x3x32  &  conv 3x3x48  &  conv 3x3x32  \\
           &  pool 4x4     &               &  conv 3x3x16  &  conv 3x3x32  &  conv 3x3x64  &  conv 3x3x32  \\
           &               &               &  pool 2x2     &  pool 2x2     &  conv 3x3x64  &  conv 3x3x48  \\
           &               &               &               &               &  pool 2x2     &  conv 3x3x48  \\
           &               &               &               &               &               &  pool 2x2     \\
\hline
conv 3x3x48  &  conv 3x3x16  &  conv 3x3x96  &  conv 3x3x32  &  conv 3x3x48  &  conv 3x3x80  &  conv 3x3x80  \\
pool 4x4     &  conv 3x3x16  &  pool 4x4     &  conv 3x3x32  &  conv 3x3x64  &  conv 3x3x80  &  conv 3x3x80  \\
           &  pool 4x4     &               &  conv 3x3x32  &  conv 3x3x80  &  conv 3x3x80  &  conv 3x3x80  \\
           &               &               &  pool 2x2     &  pool 2x2     &  conv 3x3x80  &  conv 3x3x80  \\
           &               &               &               &               &  pool 2x2     &  conv 3x3x80  \\
           &               &               &               &               &               &  conv 3x3x80  \\
           &               &               &               &               &               &  pool 2x2     \\
\hline
conv 3x3x48  &  conv 3x3x16  &  conv 3x3x96  &  conv 3x3x48  &  conv 3x3x96  &  conv 3x3x128 &  conv 3x3x128  \\
pool 4x4     &  conv 3x3x16  &  pool 4x4     &  conv 3x3x48  &  conv 3x3x96  &  conv 3x3x128 &  conv 3x3x128  \\
           &  pool 4x4     &               &  conv 3x3x64  &  conv 3x3x128 &  conv 3x3x128 &  conv 3x3x128  \\
           &               &               &  pool 8x8     &  pool 8x8     &  pool 8x8     &  conv 3x3x128  \\
           &               &               &               &               &               &  conv 3x3x128  \\
           &               &               &               &               &               &  conv 3x3x128  \\
           &               &               &               &               &               &  pool 8x8      \\
\hline
fc  &  fc  &  fc  &  fc  &  fc  &  fc  &  fc  \\
softmax  &  softmax  &  softmax  &  softmax  &  softmax  &  softmax  &  softmax  \\
\hline
\end{tabular}
\label{tab_arch_4}
\end{table*}

\subsection{Comparison with State-of-the-Art Methods}
Although the main purpose of this paper is to improve the robustness of the student networks, instead of focusing on performance of the student networks on clean data. We also compared the proposed approach with state-of-the-art teacher-student learning methods on clean datasets. For the clean data, the proposed algorithm can still achieve comparable accuracy as compared to others in Tables \ref{tab_mnist} and \ref{tab_cifar}. \Tref{tab_mnist} summarized the obtained results on three datasets: MNIST, CIFAR-10 and CIFAR-100. On the MNIST dataset, the teacher network got a 99.45\% accuracy. With the assistance of KD, the student network achieved a 99.46\% accuracy. FitNet generated a slightly better student network with a 99.49\% accuracy, which has outperformed the teacher network. Though the proposed algorithm aims to enhance the robustness of the learned student network, it can also achieve comparable or even better accuracy than those of state-of-the-art methods. The accuracy obtained by the proposed method increased to 99.51\% on the MINIST dataset. 
\Tref{tab_cifar} shows the results on the CIFAR-10 datasets, the baseline teacher network achieved a 90.25\% accuracy, and the accuracy of the student network generated by KD and Fitnet were 91.07\% and 91.64\%, respectively. The Robust student network obtained a 91.63\% accuracy, which outperforms the other student networks and teacher. This suggests that the proposed method is able to enhance the stability of student network and then improve the performance of the network.

CIFAR-100 is similar but more challengeable than CIFAR-10 because of its 100 categories. The accuracy obtained by the teacher network is only 63.49\%. As comparison, the accuracy of teacher on CIFAR-10 is 90.25\%, which is much better than that on CIFAR-100. The robust student network achieved a 65.28\% accuracy, which outperforms student networks trained by other strategies, i.e., the network trained by knowledge distillation obtains a test accuracy of 64.13\%, and the accuracy of FitNet is 64.86\%. When compared to other methods, the student network generated by the proposed method provides nearly the state-of-the-art performance. This result demonstrates that the proposed method succeeds in assisting to learn a student network with considerable performance.

\subsection{Analysis on Structures of Student Network}
We followed the experimental setting in FitNet~\cite{romero2014fitnets} and designed four student networks with different configurations of parameters and layers. The teacher network had the same structure as that used on the CIFAR-10 dataset.We design four student networks of different sizes and structures, the detailed structure of these networks can be found in \Tref{tab_arch_4}. From `Student 1' to `Student 4', the volume of the network has gradually increased, and the performance of the network has gradually increased, too. \Tref{tab_cifar_class} reported the performance of four student networks and the teacher network on the CIFAR-10 dataset. The compression ratio and speed-up ratio compared with the teacher, and the number of parameters and multiplications can also be found in \Tref{tab_students_result}.

From \Tref{tab_cifar_class}, we find that the proposed robust student network outperforms FitNet under all four different student structures. Though there is no perturbation on the data, the proposed method can achieve higher accuracy, which indicates the effectiveness of encouraging the student network to make confident predictions with the help of the teacher network. In addition, the smallest network Student 1 has the biggest compression and speed-up ratios, but it can still achieve a test accuracy of 89.62\%, which is fairly close to the 90.25\% of teacher and outperforms the 89.07\% obtained by FitNet. As Student 1 contains significantly fewer parameters than those of the teacher, improving the accuracy of such a network with limited capacity is challenging, which in turn suggests the effectiveness of the proposed method.

Although there are significantly fewer parameters contained in student networks that {}learned by the proposed method. These student networks are still regular networks which can be further compressed and speeded-up by existing sparsity based deep neural network compression technologies, such as deep compression~\cite{han2015deep} and feature compression~\cite{wang2016cnnpack}.

\section{Conclusion}  
We proposed to learn a robust student network with the guidance of the teacher network. The proposed method prevented the student network from being disturbed by the perturbations on input examples. Through a rigorous theoretical analysis, we proved a lower bound of perturbations that will weaken the student network's confidence in its prediction. We introduced new objectives based on prediction score and gradients of examples to maximize this lower bound and then improved the robustness of the learned student network to resist perturbations on examples. Experimental results on several benchmark datasets demonstrate the proposed method is able to learn a robust student network with satisfying accuracy and compact size.

\ifCLASSOPTIONcaptionsoff
  \newpage
\fi

\bibliographystyle{IEEEtran}
\bibliography{IEEEabrv,RobustGty}

\end{document}